\documentclass[pdfTeX,12pt]{article}
\usepackage[top=20truemm,bottom=20truemm,left=20truemm,right=20truemm]{geometry}
\usepackage{amsmath,amsthm,amssymb}
\usepackage{algpseudocode,algorithm}
\usepackage{wrapfig}
\usepackage{color}
\usepackage{graphicx}
\usepackage{tikz}

\theoremstyle{definition}
\newtheorem{theorem}{Theorem}
\newtheorem*{theorem*}{Theorem}

\newtheorem*{definition*}{Definition}
\newtheorem{lemma}[theorem]{Lemma}
\newtheorem*{lemma*}{Lemma}
\newtheorem{assumption}[theorem]{Assumption}
\newtheorem{assumption*}{Assumption}
\newtheorem{proposition}[theorem]{Proposition}
\newtheorem*{proposition*}{Proposition}
\newtheorem{corollary}[theorem]{Corollary}
\newtheorem*{corollary*}{Corollary}
\newcommand{\tval}[1]{\theta({#1})}
\newcommand{\eval}[1]{\eta({#1})}
\newcommand{\tvali}[1]{\theta_i({#1})}
\newcommand{\evali}[1]{\eta_i({#1})}
\newcommand{\tvalix}[2]{\theta_{#1}({#2})}
\newcommand{\evalix}[2]{\eta_{#1}({#2})}
\newcommand{\qtrue}{q^*}
\newcommand{\qcur}{\hat q}
\newcommand{\wcur}[1]{\hat w_{#1}}
\newcommand{\wtrue}[1]{w^*_{#1}}
\newcommand{\wtruebf}[1]{\mathbf{w}^*_{#1}}
\newcommand{\sumid}{\sum_{i=1}^d}
\newcommand{\sumidn}{\sum_{i=1}^{d-1}}
\newcommand{\sumkK}{\sum_{k=1}^K}
\newcommand{\sumlK}{\sum_{l=1}^K}
\newcommand{\expandDe}[2]{\psi(\tval{#1}) - \psi(\tval{#2}) + \sumid
 \evali{#2}(\tvali{#2}-\tvali{#1})}
\newcommand{\expandDm}[2]{-\phi(\eval{#1}) + \phi(\eval{#2}) + \sumid
 \tvali{#1}(\evali{#1}-\evali{#2})}
\newcommand{\dgamma}[1]{\frac{df(#1)}{d\gamma}}
\newcommand{\Alg}[2]{Algorithm #1($#2$)}
\newcommand{\dual}[1]{\tilde{#1}}

\title{On a convergence property of a geometrical
algorithm for statistical manifolds\thanks{This is a full version of the paper
presented in ICONIP2019}}
\author{
Shotaro Akaho$^{1,4}$
\and
Hideitsu Hino$^{2,4}$
\and
Noboru Murata$^{3,4}$ \\
$^1$National Institute of Advanced Industrial Science and Technology, \\
Tsukuba, Ibaraki 305-8568, Japan \and
$^2$The Institute of Statistical Mathematics, Tachikawa, Tokyo, 190-8562, Japan
\and
$^3$Waseda University, Shinjuku, Tokyo 169-0072, Japan
\and
$^4$RIKEN Center for Advanced Intelligence Project, Chuo, Tokyo 103-0027, Japan
} 
\date{\today}

\begin{document}

\maketitle

\abstract{In this paper, we examine a geometrical projection algorithm for statistical inference.
The algorithm is based on Pythagorean relation and it is derivative-free as well as representation-free that is useful in nonparametric cases. 
We derive a bound of learning rate to guarantee local convergence.
In special cases of m-mixture and e-mixture estimation problems,
we calculate specific forms of the bound that can be used easily in practice.}

\section{Introduction}

Information geometry is a framework to analyze statistical inference
and machine learning\cite{amari2016}.
Geometrically, statistical inference and many machine learning
algorithms can be regarded as procedures to
find a projection to a model subspace from a given data point.
In this paper, we focus on an algorithm to find the projection.

Since the projection is given by minimizing a divergence,
a common approach to finding the projection is a gradient-based method\cite{fujiwara1995}.
However, such an approach is not applicable in some cases.
For instance, several attempts to extend the information geometrical
framework to nonparametric cases\cite{ay2017information,lebanon2015riemannian,pistone2013,takano2016},
where we need to consider a function space or each data is represented as
a point process. In such a case, it is difficult to compute
the derivative of divergence that is necessary for gradient-based methods, and
in some cases, it is difficult to deal with the coordinate explicitly.

Takano et al.\cite{takano2016} proposed a geometrical algorithm
to find the projection for nonparametric e-mixture distribution,
where the model subspace is spanned by several empirical distributions.
The algorithm that is derived based on the generalized Pythagorean theorem
only depends on the values of divergences.
It is derivative-free as well as representation-free, 
and it can be applicable to many machine learning
algorithms that can be regarded as finding a projection,
but its convergence property has not been analyzed yet.
The first contribution of this paper is to extend the algorithm
to more general cases.
The second contribution is to give a condition for the convergence
of the algorithm, which is given as a bound of learning rate. In the case of
the discrete distribution, we obtain specific forms of the bound that can be used
easily in practice.

\section{Geometrical algorithm}

\subsection{Projection in a statistical manifold}

Here we briefly review the information
geometry in order to explain the proposed
geometrical algorithm based on generalized Pythagorean theorem\cite{nagaoka1982}.

Let $(S,g,\nabla,\dual{\nabla})$ be a statistical manifold,
where $S$ is a smooth manifold with a Riemannian metric $g$,
dual affine connections $\nabla$ and $\dual{\nabla}$.
We consider the case that $S$ is (dually) flat, where there exist
a $\nabla$-affine coordinate $\theta=(\theta_1,\ldots,\theta_d)$
and a $\dual{\nabla}$-affine
coordinate $\eta=(\eta_1,\ldots,\eta_d)$.
For a flat manifold, there exist potential functions
$\psi(\theta)$ and $\phi(\eta)$, and the two coordinates $\theta$
and $\eta$ are transformed each other by
Legendre transform,
\begin{equation}
 \theta_i = \frac{\partial\phi(\eta)}{\partial\eta_i}, \quad
 \eta_i = \frac{\partial\psi(\theta)}{\partial\theta_i}, \quad
 \psi(\theta)+\phi(\eta)-\sumid \theta_i\eta_i=0.
\end{equation}
A typical example of a flat manifold is an exponential family,
where each member of the manifold is a distribution of a random variable $x$
with parameter $\xi=(\xi_1,\ldots,\xi_d)$,
\begin{equation}
\label{eq:exp}
 p(x;\xi) = \exp\left(\sumid \xi_i F_i(x) - b(\xi)\right),
\end{equation}
where $F_i(x)$ is a sufficient statistics and $\exp(-b(\xi))$ is a normalization
factor.
For the exponential family, there are two dual connections, called e-connection
and m-connection (e: exponential, m: mixture).
If we take the e-connection as the $\nabla$-connection, $\nabla$-affine
coordinate $\theta$ is equal to $\xi$ called e-coordinate,
and $\dual{\nabla}$-affine coordinate called m-coordinate is given by
\begin{equation}
 \zeta_i = \frac{\partial b(\xi)}{\partial\xi_i} =
  \mathrm{E}_\xi [F_i(x)] = \int F_i(x) p(x;\xi)dx, 
\end{equation}
where the function $b(\xi)$ becomes a potential function
$\psi(\theta)$. Note that if we take the m-connection as $\nabla$,
the relation changes in a dual way, i.e.,
$\zeta$ becomes $\theta$ and $\xi$ becomes $\eta$.

Here, for $p\in S$,
we denote the corresponding $\nabla$- and $\dual{\nabla}$-coordinate
by $\tval{p}$ and $\eval{p}$ respectively.
Let us consider a submanifold defined by linear combination of
$K$ points $p_1,\ldots, p_K\in S$,
\begin{equation}
 \label{eq:Me}
 M = \{p \mid \tval{p} =\sumkK w_k \tval{p_k}, \sumkK w_k=1\},
\end{equation}
where $\mathbf{w}=(w_1,\ldots,w_K)$ is a weight vector whose sum is 1. 
The submanifold $M$
 is an affine subspace and hence it is called an
$\nabla$-autoparallel (or $\nabla$-flat) submanifold.
In particular, if $K=2$, $M$ is a straight line of $\nabla$-coordinate
that is called $\nabla$-geodesic.

We can also consider another submanifold in the dual coordinate,
\begin{equation}
 \label{eq:Mm}
\dual{M} = \{p \mid \eval{p} = \sumkK w_k \eval{p_k}, \sumkK w_k=1\},
\end{equation}
which is called a $\dual{\nabla}$-autoparallel (or $\dual{\nabla}$-flat) submanifold.
The $\dual{\nabla}$-geodesic is defined by a straight line of $\dual{\nabla}$-coordinate.

Now let us define a $\nabla$-projection and a $\dual{\nabla}$-projection
from a point $q\in S$ onto
a submanifold $M$.
The $\nabla$-projection is a point $\qtrue\in M$ such that
$\nabla$-geodesic between $q$ and
$\qtrue$ is orthogonal to $M$ at $\qtrue$ with respect to the Riemannian metric
$g_{ij}(\tval{\qtrue})$. In the statistical manifold, $g_{ij}$ is taken as
\begin{equation}
 g_{ij}(\theta) 
  = \frac{\partial^2 \psi(\theta)}{\partial\theta_i\partial\theta_j},
\end{equation}
which is equal to Fisher information for exponential family
\begin{equation}
g_{ij}(\xi)=\mathrm{E}_\xi \left[\frac{\partial\log p(x;\xi)}{\partial\xi_i}
		    \frac{\log p(x;\xi)}{\partial\xi_j}\right].
\end{equation}
In a similar way, $\dual{\nabla}$-projection onto a submanifold $M$
is defined as a point $\qtrue$
so that the $\dual{\nabla}$-geodesic connecting $q$ and $\qtrue$ is
orthogonal to $M$.

\begin{theorem}[Generalized Pythagorean theorem\cite{nagaoka1982}]
 Let $\dual{M}$ be a $\dual{\nabla}$-autoparallel submanifold of a statistical manifold $S$,
 and the $\nabla$-projection be $\qtrue\in\dual{M}$ from a point $q\in S$, then
 for any point $p\in \dual{M}$, the following relation holds
 \begin{equation}
\label{eq:pythagorase}
  D(p,q) = D(\qtrue,q) + D(p,\qtrue),
 \end{equation}
 where $D$ is the canonical divergence defined by
 \begin{align}
  D(p,q) &= \psi(\tval{q})+\phi(\eval{p})-\sumid\evali{p}\tvali{q} \nonumber \\
  &=\expandDe{q}{p} \nonumber \\
  &=\expandDm{q}{p}.
 \end{align}
\end{theorem}
By exchanging $\nabla$ and $\dual{\nabla}$, we have a dual relation, i.e,
for a $\nabla$-autoparallel submanifold $M$,
 the $\dual{\nabla}$-projection $\qtrue\in M$ from a point $q\in S$ satisfies
 the relation
 \begin{equation}
 \dual{D}(p,q) =\dual{D}(\qtrue,q) +\dual{D}(p,\qtrue),
 \end{equation}
 where $p\in M$ and
 $\dual{D}$ is a dual divergence defined by $\dual{D}(p, q) = D(q, p)$.

From this theorem, we see that a $\nabla$-projection ($\dual{\nabla}$-projection)
onto a $\dual{\nabla}$-autoparallel ($\nabla$-autoparallel respectively)
submanifold is unique
and can be found by minimizing corresponding divergence, i.e.,
the $\nabla$-projection is given by
\begin{equation}
 \qtrue = \arg\min_{p\in \dual{M}} D(p, q)
\end{equation}
and the $\dual{\nabla}$-projection is given by
\begin{equation}
 \qtrue = \arg\min_{p\in M} \dual{D}(p, q).
\end{equation}

For the exponential family (\ref{eq:exp}), taking the e-connection as
$\nabla$-connection, the divergence is equal to the Kullback-Leibler divergence,
\begin{equation}
 D(p, q) = \int p(x;\xi(p)) \log\frac{p(x;\xi(p))}{p(x;\xi(q))}dx.
\end{equation}
If we take the e-connection as $\nabla$ or $\dual{\nabla}$ connection, the
corresponding projection and autoparallel submanifold is called
an e-projection and an e-autoparallel submanifold,
and similarly, an m-projection and an m-autoparallel submanifold are
defined for the m-connection.

\subsection{Geometrical algorithm for projection}
\label{sec:geo}
Now we propose a geometrical algorithm to find a $\nabla$-projection
(or $\dual{\nabla}$-projection) onto a $\dual{\nabla}$-autoparallel
(and $\nabla$-autoparallel respectively)
submanifold.
To avoid redundant description, we only formulate the $\nabla$-projection onto
a $\dual{\nabla}$-autoparallel submanifold,
since the dual case can be obtained by only exchanging $\nabla$ and
$\dual{\nabla}$.

In this paper, we impose a restriction on the projection.
\begin{assumption}
\label{assumption}
 The projection belongs to the convex hull
	    of $p_1,\ldots,p_K$ in
(\ref{eq:Me}) and (\ref{eq:Mm}), i.e., all $w_k>0$.
\end{assumption}
Although the projection from a point $q\in S$ does not necessarily
belong to the convex hull of basis vectors in general,
some application such as mixture models that will be
explained
in Sec.~\ref{sec:nmf} requires this assumption.
We will discuss this restriction in Sec.~\ref{sec:discussassumption}.

\begin{figure}[tbhp]
\centering
\includegraphics[width=8cm, bb=0 0 502 370]{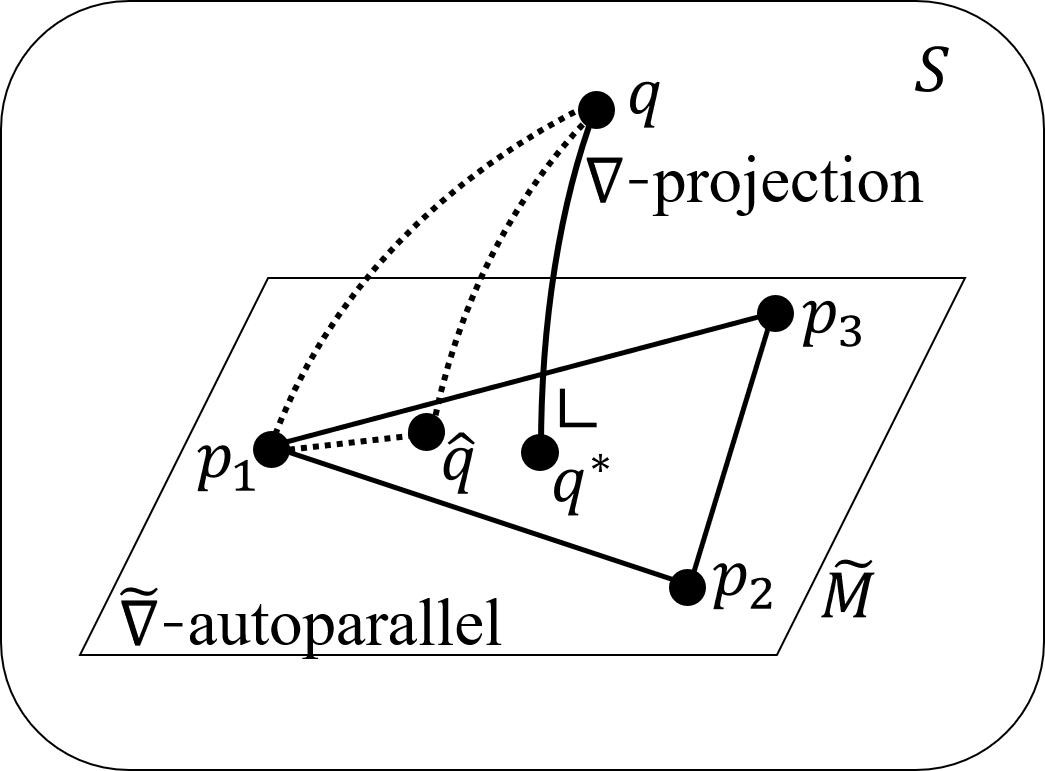}
\caption{The $\nabla$-projection $q^*$ 
from a point $q$ to an $\dual{\nabla}$-autoparallel manifold $\dual{M}$
spanned by $\{p_k\}$, where
$\qcur$ is a current estimate of $\qtrue$. The value $\gamma_k$ defined in (\ref{eq:gammak}) represents
the deviation from Pythagorean relation, i.e., $\gamma_k=0$ iff $\qcur=\qtrue$,
and $\gamma_k>0$ implies $\qcur$ is closer to $p_k$ while $\gamma_k<0$ implies
$\qcur$ is further to $p_k$.
}
\label{fig:projection}
\end{figure}

Suppose a point $q\in S$ and a $\dual{\nabla}$-autoparallel
submanifold $\dual{M}\subseteq S$ are given,
let $\qcur\in \dual{M}$ be a current estimate of the projection $\qtrue\in \dual{M}$ (Fig.~\ref{fig:projection}) and
let us define the quantity $\gamma_k$,
\begin{equation}
\label{eq:gammak}
 \gamma_k =   D(\qcur, q) + D(p_k, \qcur) - D(p_k, q).
\end{equation}
From Eq.~(\ref{eq:pythagorase}), $\gamma_k=0$ if and only if $\qcur = \qtrue$.
If $\gamma_k < 0$, that means $\qcur$ is closer to $p_k$ than $\qtrue$,
$w_k$ should be decreased.
On the other hand, if $\gamma_k > 0$, 
$\qcur$ is farther from $p_k$ than $\qtrue$, $w_k$ should be increased.

From the consideration above, we can construct the Algorithm \ref{alg:AK}
to find the $\nabla$-projection by optimizing weights $\{w_k\}_
{k=1,\ldots,K}$ so that $\qcur$ satisfies the Pythagorean relation (\ref{eq:pythagorase}).

\begin{algorithm}
\caption{Geometrical \Alg{A}{K}} \label{alg:AK}
\begin{algorithmic}[1]
 \State \textbf{Initialize} $\{w_k^{(0)}\}_{k=1,\ldots,K}$
   s.t.\ $\sumkK w_k^{(0)} = 1, w_k^{(0)} > 0, t:=0$
 \Repeat
 \State Calculate $\gamma_k$ by (\ref{eq:gammak}), where
       $\eval{\qcur}=\sum_{i=1}^K w_k^{(t)} \eval{p_k}$, $k=1,\ldots,K$
 \State Update $w_k$, $k=1,\ldots,K$  by
       \begin{equation}
\label{eq:updatew}
	w'_k = w_k^{(t)} f(\gamma_k)
       \end{equation}
 \State Normalize $w_k'$, $k=1,\ldots,K$ by
       \begin{equation}
\label{eq:normalw}
	w_k^{(t+1)} = \frac{w_k'}{\sumkK w_k'}
       \end{equation}
 \State $t:=t+1$
\Until{Stopping criterion is satisfied}
\State \Return $\mathbf{w}$
\end{algorithmic}
\end{algorithm}

In the algorithm, the function $f(\gamma)$ is a positive and strictly
       monotonically increasing function
       s.t.\ $f(0)=1$, which is
       introduced in order to stabilize the algorithm and a typical
choice of $f$ is a sigmoidal function,
\begin{equation}
 f(\gamma) = \frac{2}{1+\exp(-\beta \gamma)}, \quad \beta>0.
\end{equation}
A parameter $\beta$ controls the learning speed and it is related
to convergence characteristics of the algorithm.
\Alg{A}{K} in the case that m-connection is taken as $\nabla$-connection
was firstly introduced by Takano et al.\cite{takano2016} in order to
estimate a nonparametric e-mixture distribution.
The main contribution of this paper is to clarify the relation between
the function $f$ and the convergence property.
In later sections, we prove \Alg{A}{2} (and also \Alg{A}{K})
is locally stable if the derivative of $f$ at the origin is less than
a certain bound. 
For later theoretical analysis, we show the following Lemma here.
\begin{lemma}\label{lemma:gamma}
The value $\gamma_k$ in \Alg{A}{K} is given by 
\begin{equation}
 \label{eq:gammak2}
 \gamma_k = \sumid (\tvali{\qcur}-\tvali{\qtrue})(\evali{\qcur}-\evali{p_k}),
\end{equation}
which means that
$\gamma_k$ only depends on the points on $\dual{M}$, if the true projection $\qtrue$ is known.
\end{lemma}

\begin{proof}
For any $p, r\in \dual{M}$ and $q\in S$, let us define
\begin{align}
\gamma(p,q,r) &= D(r,q) + D(p, r) - D(p, q) \nonumber \\
 &= \expandDm{q}{r} \nonumber \\
 &\quad  \expandDm{r}{p} \nonumber \\
 &\quad - \left\{\expandDm{q}{p}\right\} \nonumber \\
 &= \sumid (\tvali{r}-\tvali{q})(\evali{r}-\evali{p}).
\end{align}
The value $\gamma_k$ is given by
\begin{align}
 \gamma_k =& \gamma(p_k, q, \qcur) = 
  \sumid (\tvali{\qcur}-\tvali{q})(\evali{\qcur}-\evali{p_k})
 \nonumber \\
  &= \sumid (\tvali{\qcur}-\tvali{\qtrue}
  + \tvali{\qtrue}-\tvali{q})(\evali{\qcur}-\evali{p_k})
 \nonumber \\
  &= \sumid (\tvali{\qcur}-\tvali{\qtrue})(\evali{\qcur}-\evali{p_k})
 \nonumber \\
  &\quad + \sumid
 (\tvali{\qtrue}-\tvali{q})
 (\evali{\qcur}-\evali{\qtrue}+\evali{\qtrue}-\evali{p_k})
 \nonumber \\
 &=\sumid (\tvali{\qcur}-\tvali{\qtrue})(\evali{\qcur}-\evali{p_k}) 
- \gamma(\qcur,q,\qtrue) + \gamma(p_k, q, \qtrue).
\end{align}
From the Pythagorean theorem,
 \begin{equation}
  \gamma(\qcur, q, \qtrue) = \gamma(p_k, q, \qtrue) = 0,
 \end{equation}
thus $\gamma_k$ becomes (\ref{eq:gammak2}).
\end{proof}

\section{Stability analysis in the case of $K=2$}

We start the analysis from the simplest case of $K=2$.
As shown later, the case of general $K$ is reduced to this case.
From (\ref{eq:gammak2}), $\gamma_k$ is only depends on the points
on the $\dual{M}$, and if $K=2$, $\dual{M}$ is just a one-dimensional
straight line of $\eta$.

\subsection{Behavior of $\gamma_k$}

In order to derive the condition for convergence, we examine the behavior
of $\gamma_k$ for a small perturbation.

The weight value $w_1$ can be regarded as an $\dual{\nabla}$-coordinate of $\dual{M}$,
and let $u_1$ be the $\nabla$-coordinate that is dual to $w_1$.
Let $\wtrue{1}$ be the value of $w_1$ at the projection point $\qtrue$,
and the current estimation $\wcur{k}(=w_k^{(t)})$
is perturbed slightly from $\wtrue{k}$,
\begin{equation}
\wcur{1} = \wtrue{1}+\epsilon, \quad \wcur{2}=\wtrue{2}-\epsilon=
 (1-\wtrue{1})-\epsilon,
\end{equation}
then from (\ref{eq:gammak2}), the value $\gamma_1$ is given by
\begin{equation}
 \gamma_1 = (w_1(\qcur)-w_1(p_1))(u_1(\qcur)-u_1(\qtrue)) =
  (\wtrue{1}+\epsilon-1)\Delta u_1,
\end{equation}
where $w_1(\qtrue)=\wtrue{1}$ and $w_1(p_1)=1$ are the $w_1$ value at
$\qtrue$ and $p_1$ respectively, and
\begin{equation}
\Delta u_1 = u_1(\qcur)-u_1(\qtrue).
\end{equation}
When $\epsilon$ is small, it can be expanded upto the first order of $\epsilon$,
\begin{equation}
 \Delta u_1 = g(w_1) \Delta w_1 + o(\epsilon),
\end{equation}
where
\begin{equation}
 \Delta w_1 = \wcur{1} - \wtrue{1} = \epsilon,
\end{equation}
 $g(w_1)$ is Jacobian that is equal to Riemannian metric
\begin{equation}
 g(w_1) = \frac{\partial u_1}{\partial w_1},
\end{equation}
and is also obtained by
\begin{equation}
 g(w_1)
  = \mathrm{E}_{w_1} \left[\left(\frac{\partial\log p(x;w_1)}{\partial w_1}
		    			   \right)^2\right]
  = -\mathrm{E}_{w_1} \left[\frac{\partial^2\log p(x;w_1)}{\partial w_1{}^2}
		    			   \right].
\end{equation}
As a result, we have
\begin{equation}
\label{eq:gamma1}
 \gamma_1 = g(\wtrue{1})(\wtrue{1}-1)\epsilon + o(\epsilon).
\end{equation}

Similarly,
\begin{equation}
 \gamma_2 = g(\wtrue{1})\wtrue{1}\epsilon + o(\epsilon).
\end{equation}

\subsection{The condition for local stability of the \Alg{A}{2}}
\label{sec:stability2}

In this section, we show the condition for local convergence property
of the \Alg{A}{2}. Here we call the algorithm is locally stable 
when the amount of sufficiently small perturbation from the optimal
solution is decreased by the algorithm.

\begin{theorem}
\Alg{A}{2} is locally stable when it holds
\begin{equation}
\label{eq:bound0}
\dgamma0 < \frac{2}{\wtrue{1}(1-\wtrue{1}) g(\wtrue{1})},
\end{equation}
where $\wtrue{1}$ is the optimal weight.
\end{theorem}

\begin{proof}
By the \Alg{A}{2}, the weight $w_1$ is updated by
\begin{equation}
 w_1' = \wcur{1}f(\gamma_1),
\end{equation}
and its first order expansion is given from Eq.~(\ref{eq:gamma1}) by
\begin{align}
 w_1' &= \wcur{1}(1 + \dgamma0\gamma_1) + o(\epsilon) \nonumber \\
 &=\wtrue{1} + \epsilon + \wtrue{1} \dgamma0 g(\wtrue{1})
 (\wtrue{1}-1)\epsilon + o(\epsilon), \label{eq:updatew1}
\end{align}
and for $w_2$,
\begin{align}
\label{eq:updatew2}
 w_2' &= \wcur{2}f(\gamma_2) \nonumber \\
 &= 1-\wtrue{1} - \epsilon +
 (1-\wtrue{1})g(\wtrue{1})\wtrue{1}\dgamma0\epsilon + o(\epsilon).
\end{align}

We see that $w_1' + w_2'=1+o(\epsilon)$, thus the normalization
procedure is negligible up to the first order of $\epsilon$.

The condition that $\qtrue$ is a stable point of the algorithm is
given by 
\begin{equation}
 |w_1' - \wtrue{1}| < |\wcur{1}-\wtrue{1}|=|\epsilon|.
\end{equation}
From Eq.~(\ref{eq:updatew1}), it is
\begin{equation}
 |\epsilon + \wtrue{1} \dgamma0 g(\wtrue{1})
 (\wtrue{1}-1)\epsilon| < |\epsilon|,
\end{equation}
which is equivalent to
\begin{equation}
\wtrue{1} \dgamma0 g(\wtrue{1})
 (1-\wtrue{1}) < 2,
\end{equation}
then we have
\begin{equation}
\dgamma0 < \frac{2}{\wtrue{1} (1-\wtrue{1}) g(\wtrue{1}) }.
\end{equation}
\end{proof}

Since the true value $\qtrue$ is not known when the algorithm is
applied, we have two approaches.
The one is approximating $\wtrue{1}$ by the current estimate $\wcur{1}$
and use adaptively changing the derivative of $f$,
which will be examined in sec.\ref{sec:adaptive}.
The other approach is to use a bound that is independent of $\wtrue{1}$,
which is available in some special cases.

\begin{corollary}
\Alg{A}{2} is locally stable when it holds
\begin{equation}
\label{eq:bound}
\dgamma0 < \frac{2}{\sup_{w} w (1-w) g(w)},
\end{equation}
 where we denote $w=w_1$ for simplicity.
\end{corollary}

\section{Special case: discrete distribution}

In the following subsections, 
we give specific forms of the bound $df(0)/d\gamma$ of Eq.~(\ref{eq:bound})
both for the e-projection and m-projection
by considering a discrete distribution as a specific case.

The discrete distribution is given by
\begin{equation}
 q(x) = \sumid q_i \delta_i(x), \quad x\in \{1,2,\ldots,d\},
\end{equation}
\begin{equation}
 \sumid q_i = 1, \quad q_i\ge0.
\end{equation}
where $\delta_i(x)=1$ when $x=i$ and $\delta_i(x)=0$ otherwise.
We see that the discrete distribution belongs to the exponential family as follows:
\begin{align}
 q(x) &= \exp\left(\sumid (\log q_i) \delta_i(x)\right) \nonumber \\
 &= \exp\left(\sumidn (\log q_i) \delta_i(x) +
 (\log q_d)
 \left(1-\sumidn \delta_i(x)\right)\right) \nonumber \\
 &= \exp\left(\sumidn \log\frac{q_i}{q_d}\delta_i(x)
  + \log q_d\right),
\end{align}
where we have $d-1$ independent parameters $q_1,\ldots,q_{d-1}$ and
one dependent parameter $q_d$ is given by $q_d=1-\sumidn q_i$.
 By taking the e-connection as the $\nabla$-connection, $q(x)$ becomes
 the same form as
 Eq.~(\ref{eq:exp}) by regarding
 \begin{equation}
\label{eq:xi}
  F_i(x) = \delta_i(x), \quad \xi_i = \log\frac{q_i}{q_d},
   \quad b(\xi) = -\log q_d, \quad i=1,\ldots,d-1.
 \end{equation}
 The dual coordinate $\zeta_i$ is given by
 \begin{equation}
\label{eq:zeta}
  \zeta_i = \mathrm{E}_{q(x)}[F_i(x)] = q_i, \quad i=1,\ldots,d-1.
 \end{equation}
The basis vectors in $S$ are denoted by
\begin{equation}
 p_k(x) = \sumid p_{ki}\delta_i(x),\quad
 \sumid p_{ki} = 1, \quad p_{ki}\ge0.
\end{equation}

\subsection{The case of e-projection}
\label{sec:nmf}

First, we take the e-connection as the $\nabla$-connection, then the $\nabla$-projection
onto the $\dual{\nabla}$-autoparallel submanifold is the e-projection
onto the m-autoparallel submanifold.

The m-autoparallel submanifold spanned by $p_k(x)$ is given by a set
of points whose m-coordinate (\ref{eq:zeta}) is given by
\begin{equation}
 \zeta_i = \sumkK w_k p_{ki},\quad i=1,\ldots,d-1.
\end{equation}
Since $\zeta_i$ is the probability value, 
it is equivalent to the mixture distribution of $\{p_k(x)\}$
\begin{equation}
\label{eq:mix}
 p(x;\mathbf{w}) = \sumkK w_k p_k(x),\quad \sumkK w_k=1,
\end{equation}
where $w_k$ is usually assumed to be positive, which matches
the Assumption \ref{assumption}.

The mixture distribution has a lot of applications, in which
complicated distribution is decomposed into sum of
simple component distributions. An important application in the
discrete distribution case is the nonnegative matrix factorization\cite{lee1999},
where a matrix $X$ with nonnegative components is approximated by
\begin{equation}
 X \simeq D C,
\end{equation}
where $D$ and $C$ are also matrices with nonnegative components.
Let $\Pi$ be the normalization operator by which sum of each column 
components become 1. 
It is known\cite{dong2014nonnegative} that if $X=DC$,
there exist $P$ and $W$ such that
\begin{equation}
 \Pi(X) = P W,
\end{equation}
where $P$ and $W$ are matrices with nonnegative components and
sum of each column components is 1.
This means that a set of probability distributions are approximated
by mixture of factor distributions.
In the NMF, $D$ and $C$ are optimized alternatively by fixing the other.
Each optimization problem can be regarded as e-projection to m-autoparallel
manifold.

Note that we consider the e-projection onto an m-autoparallel submanifold
in this paper, since it is natural from the generalized Pythagorean
relation. However, many learning algorithms are
formulated to maximum likelihood
that is equivalent to the m-projection, which is different from
e-projection in the sense that the argument of divergence is reversed.
For the discrete distribution case, 
the m-projection to the m-autoparallel submanifold has a unique solution,
but it does not hold in general.

Now we give a sufficient condition for convergence of the e-projection
onto the m-autoparallel submanifold.
\begin{proposition}
The \Alg{A}{2}
of the e-projection onto an m-autoparallel submanifold
 for the discrete distribution locally stable
if
 \begin{equation}
 \dgamma0 < \frac{2}{\sum_i(\sqrt{p_{1i}}-\sqrt{p_{2i}})^2},
 \end{equation}
 where the right hand side has a constant lower bound $\sqrt{2}$.
\end{proposition}

\begin{proof}
The m-autoparallel model spanned by $K=2$ points can be written as
 \begin{equation}
 p(x;w) = w p_1(x) + (1-w) p_2(x).
\end{equation}
The Riemannian metric at $p(x;w)$ is given by
\begin{align}
g(w)  &=
 \mathrm{E}_{w} \left[\left(\frac{\partial\log p(x;w)}{\partial w}\right)^2\right]
 \nonumber\\
 &= \sum_{x=1}^d \frac{1}{p(x;w)}
 \left(\frac{\partial
 p(x;w)}{\partial w}\right)^2
\nonumber \\
 &= \sum_{x=1}^d  \frac{(p_1(x)-p_2(x))^2}{p(x;w)}
 \nonumber \\
 &= \sum_{i=1}^d  \frac{(p_{1i}-p_{2i})^2}{w p_{1i}+(1-w)p_{2i}}.
\end{align}
The denominator of right hand side of Eq.~(\ref{eq:bound}) is 
\begin{equation}
\label{eq:boundm}
\sup_w w (1-w) 
\sum_{i=1}^d  \frac{(p_{1i}-p_{2i})^2}{w p_{1i} +
  (1-w) p_{2i}}.
\end{equation}
The $i$-th term
\begin{equation}
w (1-w) 
\frac{(p_{1i}-p_{2i})^2}{w p_{1i} +
  (1-w) p_{2i}}
\end{equation}
has maximum value $(\sqrt{p_{1i}}-\sqrt{p_{2i}})^2$
when $w=\sqrt{p_{2i}}/(\sqrt{p_{1i}}+\sqrt{p_{2i}})$, then
Eq.~(\ref{eq:boundm}) is bounded from upper by
\begin{equation}
 \sum_i(\sqrt{p_{1i}}-\sqrt{p_{2i}})^2,
\end{equation}
 which is a Hellinger distance between $p_1(x)$ and $p_2(x)$,
and  we obtain the sufficient condition for local stability,
\begin{equation}
 \dgamma0 < \frac{2}{\sum_i(\sqrt{p_{1i}}-\sqrt{p_{2i}})^2},
\end{equation}
and the right hand side has a constant lower bound $\sqrt{2}$.
\end{proof}

\subsection{The case of m-projection}

In this subsection, we take the m-connection as the $\nabla$-connection,
then the $\nabla$-projection
onto the $\dual{\nabla}$-autoparallel submanifold is the m-projection
onto the e-autoparallel submanifold.

The e-autoparallel submanifold spanned by $p_k(x)$ is given by a set
of points whose e-coordinate (\ref{eq:xi}) is given by
\begin{equation}
 \xi_i = \sumkK w_k \log\frac{p_{ki}}{p_{kd}}
  = \left(\sumkK w_k \log p_{ki}\right) - \log p_{kd},\quad i=1,\ldots,d-1.
\end{equation}
Since $\xi_i=\log (q_i / q_d)$, it is equivalent to the model
specified by
\begin{equation}
 p(x;\mathbf{w})\propto \exp\left(\sumkK w_k \log p_k(x)\right),
  \quad \sumkK w_k=1,
\end{equation}
which is a different type of mixture, log linear mixture.

We call this type of mixture as e-mixture, while the mixture
specified by Eq.~(\ref{eq:mix}) as m-mixture.
Although the e-mixture has not been studied as intensively as the m-mixture,
it has several good properties such as maximum entropy principle.
Takano et al.\cite{takano2016} proposed a nonparametric extension of
the e-mixture and its learning algorithm based on the geometrical algorithm,
which is generalized in this paper.
In the nonparametric e-mixture estimation, the basis distributions are
expressed by the empirical distribution (i.e., sum of delta functions), 
thus the e-mixture of basis distibutions cannot mathematically defined.
Instead,  it is defined by geometrical characteristics
of e-mixture\cite{murata2009}. Therefore, it is not possible to obtain
the coordinate explicitly. Because the geometrical algorithm is coordinate-free, and it only requires to calculate divergences, which can be
estimated based on nonparametric entropy estimation\cite{Hino201572,DBLP:conf/iconip/HinoAM16}.
This is a strong motivation to propose the geometrical algorithm.

Here we give a sufficient condition for convergence of the m-projection
onto the e-autoparallel submanifold.
\begin{proposition}
The \Alg{A}{2}
of the m-projection onto the e-autoparallel submanifold
 for the discrete distribution is locally stable if
\begin{equation}
\label{eq:bounde}
 \dgamma0 \le \frac{32}{\displaystyle \left(\max_i \displaystyle \log \frac{p_{1i}}{p_{2i}} -
  \min_i \log \frac{p_{1i}}{p_{2i}}\right)^2}.
\end{equation}
\end{proposition}

The right hand side does not have a constant lower bound unlike the
e-projection
case, and it is left as an open problem whether there exists any constant
bound.

\begin{proof}
The e-autoparallel model for $K=2$ is written as
\begin{equation}
 p(x;w) = \frac{1}{Z(w)} \exp(w \log p_1(x) + (1-w) \log p_2(x)
),
\end{equation}
where $w$ is an e-coordinate, $Z(w)$ is a normalization constant
\begin{equation}
 Z(w) = \sum_{x=1}^d \exp( w \log p_1(x) + (1-w) \log p_2(x) ).
\end{equation}

Since the discrete distribution
\begin{equation}
 \log p_k(x) = \sum_{i=1}^d \log p_{ki} \delta_{i}(x),
\end{equation}
$p(x;w)$ is written as
\begin{align}
 p(x;w) &= \frac{1}{Z(w)}\exp\left(\sumid(
 w\log p_{1i} + (1-w)\log p_{2i})\delta_i(x)\right) \nonumber \\
 &=\frac{1}{Z(w)}\exp\left(\sumid (a_i w + b_i)\delta_i(x)\right),
\end{align}
where
\begin{equation}
 a_i = \log (p_{1i}/p_{2i}), \quad b_i = \log p_{2i},
\end{equation}
\begin{equation}
 Z(w) = \sumid c_i(w), \quad c_i(w) = \exp(a_i w + b_i).
\end{equation}
Note that $p(i;w) = c_i(w) / Z(w)$

The Fisher information for this model can be calculated by
\begin{align}
 g(w) &= - \mathrm{E}_w\left[\frac{\partial^2 \log p(x;w)}%
 {\partial w^2}\right] \nonumber \\
 &= \frac{1}{Z(w)}\frac{\partial^2 Z(w)}{\partial w^2}-
 \left(\frac{1}{Z(w)}\frac{\partial Z(w)}{\partial w}\right)^2
 \nonumber \\
 &= \sumid \frac{a_i^2 c_i(w)}{Z(w)} -
 \left(\sumid \frac{a_i c_i(w)}{Z(w)}\right)^2 \nonumber \\
 &= \sumid a_i^2 p(i;w) - \left(\sumid a_i p(i;w)\right)^2
\end{align}
The last formula represents the variance of $a_i$
with respect to the probability weight $p(i;w)$.
From Popoviciu's inequality on variances\cite{popoviciu1935},
$g(w)$ has an upper bound that is independent of $w$,
\begin{equation}
 g(w)\le \frac{1}{4}(\max_i a_i - \min_i a_i)^2.
\end{equation}
Since $w(1-w)\le 1/4$, we obtain the inequality (\ref{eq:bounde}) of the Proposition from Eq.~(\ref{eq:bound}).
\end{proof}

\section{Local stability for general $K$}

We proceed to the general case which include $K\ge 2$.
First we present the main theorem.

\begin{theorem}
\label{th:main}
Let $w_k^*$, $k=1,\ldots, K$ be the optimal parameter.
If the function $f$ satisfies
\begin{equation}
\label{eq:boundK}
\dgamma0 < \frac{2}{K \max_k w_k^* (1-w_k^*) g(w_k^*)},
\end{equation}
 \Alg{A}{K} is locally stable.
\end{theorem}

The proof is in the appendix.
Basic strategy of the proof is to show the equivalence between
the \Alg{A}{K} and a component-wise update algorithm based on
\Alg{A}{2}. In the process of the proof, a possible refinement of
the \Alg{A}{K} is also suggested.

\section{Discussion}

In this section, we will discuss several points related to the proposed framework, (1) relation to gradient descent method, (2) possible refinement
of the algorithm, (3) assumption of the positivity.

\subsection{Relation to gradient descent method}

In general optimization problems, a gradient descent method is a simple way to solve the problem. Here, we show that the updates of the 
gradient descent and the proposed algorithm are linearly related. 

The parameter $\{w_k\}$ should satisfy a constraint 
$\sum_{k=1}^K w_k=1$. We first replace $w_K$ by $1-\sum_{k=1}^{K-1} w_k$,
then update $w_k$ for $k=1,\ldots, K-1$ by
\begin{equation}
w_k' = w_k - \lambda \frac{\partial D(\hat{q}, q)}{\partial w_k},
\end{equation}
and $w_K'$ is obtained by $1-\sum_{k=1}^{K-1} w_k'$.

The gradient of $D(\hat{q},q)$ with respect to $w_k$ is given by
\begin{align}
\frac{\partial D(\qcur,q)}{\partial w_k} &=
\frac{\partial}{\partial w_k} \left[\expandDm{q}{\qcur}\right]\nonumber\\
&=\frac{\partial}{\partial w_k} \left[
\phi(\eval{\qcur})-\sumid \tvali{q}\evali{\qcur}\right]\nonumber\\
&= \sumid \frac{\partial \eta_i(\hat{q})}{\partial w_k}
\frac{\partial}{\partial{\eta_i(\hat{q})}}
\left[
\phi(\eval{\qcur})-\sum_{i'=1}^d \tvalix{i'}{q}
\evalix{i'}{\qcur}\right].
\end{align}
Since 
\begin{equation}
\eta(\hat{q}) = \sum_{k=1}^K w_k \eta(p_k)=\sum_{k=1}^{K-1} w_k (\eta(p_k)-\eta(p_{K})) + \eta(p_{K}),
\end{equation} 
and $\partial \phi(\eta)/\partial \eta_i=\theta_i$, we have
\begin{align}
\frac{\partial D(\hat{q},q)}{\partial w_k} &=\sumid
(\tvali{\qcur}-\tvali{q})(\evali{p_k}-\evali{p_K}),
\end{align}
which can be, from Lemma \ref{lemma:gamma}, represented using $\gamma_k$,
\begin{equation}
\frac{\partial D(\hat{q},q)}{\partial w_k} =\gamma_K-\gamma_k.
\end{equation}
The amount of update by the gradient descent is
\begin{equation}
\label{eq:DwG}
\Delta w_k^G = -\lambda  \frac{\partial D(\hat{q},q)}{\partial w_k}
=\lambda (\gamma_k - \gamma_K),
\quad k=1,\ldots, K-1,
\end{equation}
and
\begin{equation}
\label{eq:DwGK}
\Delta w_K^G = -\sum_{k=1}^{K-1} \Delta w_{k}^G = 
\lambda \sum_{k=1}^{K-1}(\gamma_K-\gamma_k).
\end{equation}
On the other hand, the amount of update by the \Alg{A}{K} is
approximated for small $\gamma_k$ by
\begin{equation}
\label{eq:DwA}
\Delta w_k^A = w_k' - w_k = w_k (f(\gamma_k)-1) 
\simeq w_k \frac{\partial f(0)}{\partial \gamma} \gamma_k.
\end{equation}
Since $\sum_k w_k \gamma_k=0$, no further normalization is
necessary.
Comparing Eq.~(\ref{eq:DwG}) and Eq.~(\ref{eq:DwGK}) with Eq.~(\ref{eq:DwA}),
we see that $\Delta w_k^G$ and $\Delta w_k^A$ are linearly related.
Unlike the gradient descent method, the proposed framework does not
need explicit calculation of the coodinate.

\subsection{Possible refinement of the algorithm}
\label{sec:adaptive}
As explained in Sec.~\ref{sec:stability2}, the condition for convergence
depends on the true parameter, thus one approach to use the adaptively
change the derivative of $f$ is to replace
the true parameter by its estimate. This approach also requires
to estimate the Fisher information.

Another possibility for the refinement of the algorithm is based
on the analysis in the appendix. It will be shown that the \Alg{A}{K}
is equivalent to the slower version of the component-wise update
algorithm. More specifically, the amount of the update $\Delta w_k$
is smaller by the factor $1-w_k$. Therefore, the update
rule in the \Alg{A}{K}, $w_k'=w_k f(\gamma_k)$ can be replaced by
\begin{equation}
w_k' = w_k f\left(\frac{\gamma_k}{1-w_k}\right),
\end{equation}
which does not change the condition of the convergence.

\subsection{On the assumption of positivity}
\label{sec:discussassumption}

In Sec.~\ref{sec:geo}, we assumed that the projection lies on the
convex hull $P\subset \dual{M}$
spanned by the basis vectors. In general, however, the projection
point can be out of $P$. In such a case, we generalize the problem
to find a point on $\dual{M}$ that minimizes the divergence,
\begin{equation}
 \qtrue = \arg\min_{p\in P} D(p, q).
\end{equation}
When the projection point is out of $P$, the solution $\qtrue$ of this problem
is on the boundary of $P$ and the $\nabla$-geodesic connecting $q$ and $\qtrue$
is not orthogonal to $\dual{M}$ any more.

The proposed algorithm itself works even in this case, because
the boundary is again a convex hull of a subset of basis vectors.
However, we have to be careful about one thing:
once a certain $w_k$ becomes 0, it cannot take positive value
any longer, which means that if the current estimate reaches to the
boundary that does not include the optimal solution, then the estimator
cannot escape from the boundary.

Without the assumption of $w_i>0$, the $\nabla$-projection of $q$ to a dual autoparallel submanifold $M$ always
exists uniquely and such a formulation is studied
as e-PCA and m-PCA framework\cite{akaho2004} or exponential family PCA
in a special case\cite{collins2001}.
However, the algorithm proposed in this paper cannot be applied as it is, because it is derived under the assumption. 
One method of update for this general case is as follows:
$w_k$ should be increased for positive $\gamma_k$ and
should be decreased for negative $\gamma_k$. Also, $\sum_k w_k=1$
should be preserved. Therefore, let $\mathcal{K}_P$ be a set of
indices with positive $\gamma_k$ and $\mathcal{K}_P$ be
a set of indices with negative $\gamma_k$. Then
\begin{align}
w_k' &= w_k + \epsilon\frac{\gamma_k}{\sum_{l\in\mathcal{K}_P}\gamma_l},
\quad k\in\mathcal{K}_P, \nonumber \\
w_k' &= w_k -  \epsilon\frac{\gamma_k}{\sum_{l\in\mathcal{K}_N}\gamma_l},
\quad k\in\mathcal{K}_N,
\end{align}
for a learning constant $\epsilon$.
There are several variations of such an update, and we also have to take
care the update does not make $w_k$ out of the domain of $w_k$.
The investigation of convergence property of the modified algorithm
is left as a future work.

\section{Concluding remarks}
We proposed a geometrical projection algorithm that only requires the
calculation of divergences. We also showed the condition of the local
stability of the algorithm. There are various applications in machine learning
and related areas in which the projection onto an autoparallel submanifold
is needed, and they are left as future works.

\section*{Acknowledgement}
This work was supported by JSPS KAKENHI Grant Numbers 
17H01793, 19K12111.

%
%
%

\bibliographystyle{plain}
\bibliography{gproj}

\setcounter{section}{1}
\setcounter{subsection}{0}

\renewcommand{\thesection}{\Alph{section}}

\section*{Appendix \ Proof of Theorem \ref{th:main}}

\begin{figure}
\centering
\includegraphics[width=10cm]{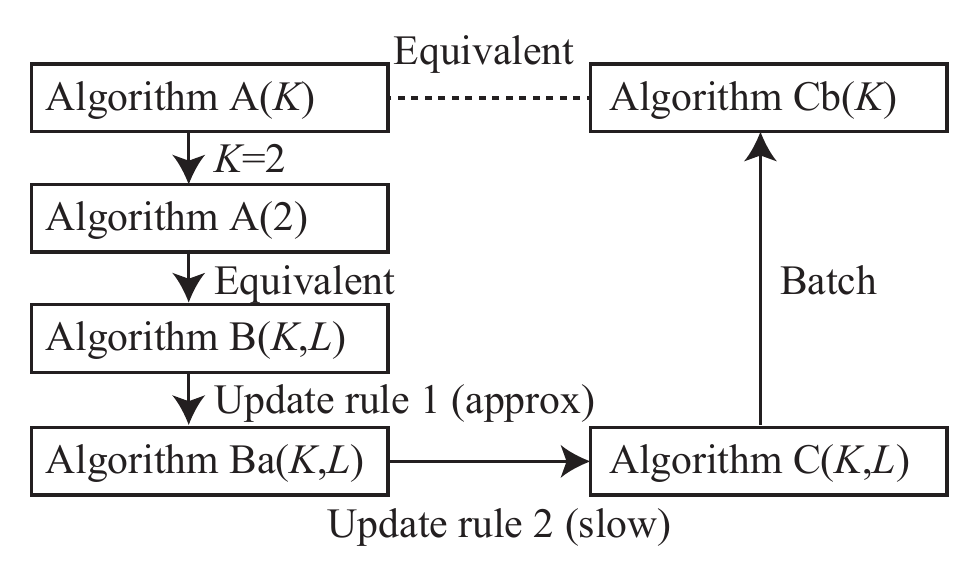}
\caption{Relation of the algorithms}
\label{fig:algorithm}
\end{figure}


\subsection{Component-wise algorithm}

\begin{figure}
\centering
\includegraphics[width=10cm]{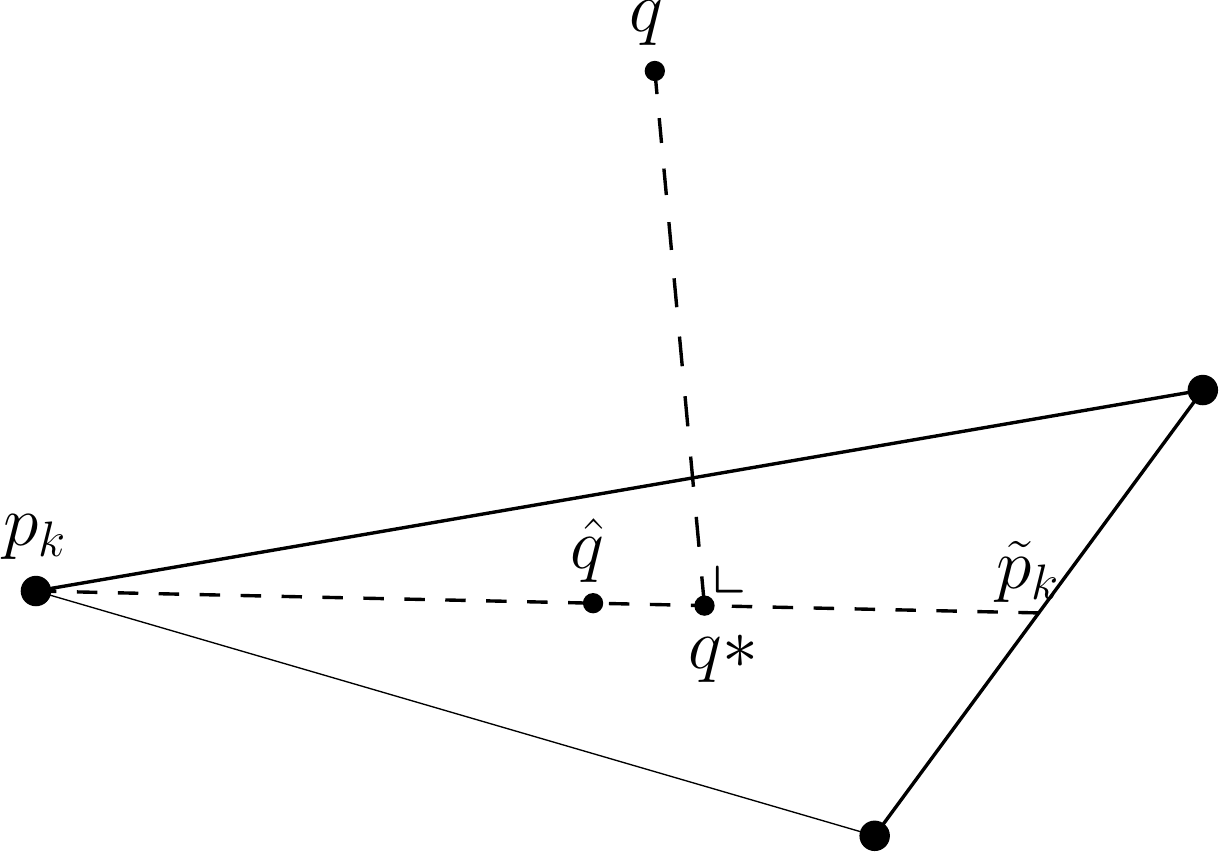}
\caption{Component-wise algorithm}
\label{fig:component-wise}
\end{figure}

Suppose the point $p_k^\dag$ be the point where
the extension line from $p_k$ through $\qcur$ intersects with
the boundary of $\dual{M}$ (Fig.~\ref{fig:component-wise}).
Let us introduce the notation $w_k(p)$ to specify the point $p$ of $\dual{M}$,
\begin{equation}
 \eval{p} = \sumkK w_k(p) \eval{p_k},
\end{equation}
for instance $w_{l}(p_k) = \delta_{k}(l)$ and
$w_k(\qcur)=\wcur{k}$. 
Since $\eval{p_k}$, $\eval{\qcur}$ and $\eval{p_k^\dag}$
are on the same line, by taking an appropriate $\rho$,
\begin{equation}
\label{eq:etadag}
 \eval{p_k^\dag} = (1-\rho) \eval{p_k} + \rho \eval{\qcur}
\end{equation}
or equivalently,
\begin{equation}
 w_{l}(p_k^\dag) = (1-\rho) w_{l}(p_k) + \rho w_{l}(\qcur).
\end{equation}
For the $p_k^\dag$, $w_k(p_k^\dag)$ should
be zero, 
\begin{equation}
  (1-\rho)w_k(p_k) + \rho \wcur{k} = 0, 
\end{equation}
that is
\begin{equation}
 \rho = \frac{1}{1-\wcur{k}},
\end{equation}
and the point $p_k^\dag$ is given by
\begin{equation}
\label{eq:pdag}
 \eval{p_k^\dag}=
  \frac{1}{1-\wcur{k}}\sum_{l\ne k}\wcur{l}\eval{p_{l}}
\end{equation}

We can consider the component-wise update
for $K=2$ by using $p_k$ and $p_k^\dag$.

\begin{algorithm}
\caption{Geometrical \Alg{B}{K,L} Component-wise} 
\begin{algorithmic}[1]
 \State \textbf{Initialize} $\{w_k\}_{k=1,\ldots,K}$
   s.t.\ $\sumkK w_k = 1, w_k > 0$
 \Repeat
 \For{$k=1,\ldots,K$}
 \State Find $p_k^\dag$
 \For{$\mbox{count}=1$ to $L$}
  \State Update $w_k$ by the \Alg{A}{2} with two basis $p_k$
       and $p_k^\dag$
 \EndFor
 \EndFor
\Until{Stopping criterion is satisfied}
\State \Return $\mathbf{w}$
\end{algorithmic}
\end{algorithm}

More detailed procedures of (a) and (b) are described later. 
In the algorithm,
the number $L$ controls how each component-wise update converges,
which plays an important role for fast convergence as will be demonstrated
in Sec.~\ref{sec:conv}.

\begin{proposition}
If the condition of local stability for $K=2$ is satisfied for all $k$,
\Alg{B}{K,L} is locally stable for any $L\ge1$.
\end{proposition}

Now let us give the procedures in \Alg{B}{K,L}.
The $K=2$ algorithm between $p_k$ and $p_k^\dag$, the current
solution $\qcur$ should be represented in the form of
\begin{equation}
 \eval{\qcur} = \omega_k \eval{p_k} + \omega_k^\dag \eval{p_k^\dag},
\label{eq:omega}
\end{equation}
where $\omega_k^\dag=1-\omega_k$, and then calculate $\gamma_k$
and $\gamma_k^\dag$ based on the Pythagorean relation, and then
apply the update (\ref{eq:updatew}) and (\ref{eq:normalw}).
From Eq.~(\ref{eq:etadag}),
\begin{equation}
 \eval{\qcur} = -\frac{1-\rho}{\rho}\eval{p_k} + \frac{1}{\rho}\eval{p_k^\dag}
  = \wcur{k} \eval{p_k} + (1-\wcur{k})\eval{p_k^\dag},
\end{equation}
then the weights for $p_k$ and $p_k^\dag$ are obtained as
$\omega_k=\wcur{k}$ and $\omega_k^\dag=1-\wcur{k}$ respectively.
The update of $\omega_k$ is written as
\begin{equation}
 \omega_k' = \omega_k f(\gamma_k), \quad
 \omega_k^\dag{}' = \omega_k^\dag f(\gamma_k^\dag),
\end{equation}
and then normalization is performed as
\begin{equation}
 \omega_k''= \frac{\omega_k'}{\omega_k'+\omega_k^\dag{}'}, \quad
 \omega_k^{\dag}{}''= \frac{\omega_k^\dag{}'}{\omega_k'+\omega_k^\dag{}'} =
 1-\omega_k''.
\end{equation}
From Eqs.~(\ref{eq:pdag}) and (\ref{eq:omega}), we have
\begin{equation}
 \eval{\qcur} = \omega_k \eval{p_k} + \frac{\omega_k^\dag}{1-\wcur{k}} \sum_{l\ne
  k} \wcur{l} \eval{p_{l}}.
\end{equation}
By updating $\omega_k$ and $\omega_k^\dag$ to
$\omega_k''$ and $\omega_k^\dag{}''$ respectively,
then the corresponding update of $\{w_l\}_{l=1,\ldots,K}$ is given by
\begin{align}
 w_k'' &= \omega_k'', \nonumber \\
 w_{l}'' & = \frac{\omega_k^\dag{}''}{1-w_k}w_{l}
 =\frac{1-w_k''}{1-w_k}w_{l}, \quad l\ne k.
\end{align}
This update requires to calculate $p_k^\dag$ (and related values),
which increases the computational complexity.
For later discussions, let us rewrite the algorithm when the amount
of update is sufficiently small.
From the discussion on the analysis of $K=2$ (Eq.(\ref{eq:updatew2})), if the update
$\Delta w_k = \Delta \omega_k = \omega_k f(\gamma_k) - \omega_k = w_k f(\gamma_k) - w_k$ is sufficiently small,
\begin{equation}
 \Delta \omega_k^\dag = \omega_k^\dag f(\gamma_k^\dag) - \omega_k^\dag \simeq
   - \Delta w_k 
\end{equation}
holds, where $\simeq$ represents the neglecting higher order terms of
$\Delta w_k$.
By this approximation, the update is simplified as follows:

\paragraph{Update rule 1:}
\begin{align}
 \wcur{k}'' &= \wcur{k} + \Delta w_k, \nonumber \\
 \wcur{l}'' &= \frac{1-\wcur{k}''}{1-\wcur{k}}\wcur{l}
 = \wcur{l}-\frac{\Delta w_k}{1-\wcur{k}}\wcur{l}, \quad l\ne k.
\end{align}

Note that calculating $p_k^\dag$ is not
necessary any longer.
Based on the Update rule 1, the algorithm is simplified.

\begin{algorithm}
\caption{Geometrical \Alg{Ba}{K,L} Component-wise approximated} 
\begin{algorithmic}[1]
 \State \textbf{Initialize} $\{w_k\}_{k=1,\ldots,K}$
   s.t.\ $\sumkK w_k = 1, w_k > 0$
 \Repeat
 \For{$k=1,\ldots,K$}
 \For{$\mbox{count}=1$ to $L$}
  \State Update $w_k$ by \textbf{Update rule 1}
 \EndFor
 \EndFor
\Until{Stopping criterion is satisfied}
\State \Return $\mathbf{w}$
\end{algorithmic}
\end{algorithm}

\Alg{Ba}{K,L} behaves similarly to \Alg{B}{K,L} locally and
it requires smaller computation cost.
\begin{proposition}
If the condition of local stability for $K=2$ is satisfied for all $k$,
\Alg{Ba}{K,L} is locally stable for any $L\ge1$.
\end{proposition}

\subsection{One-side component-wise update}

The component-wise update without any approximation
requires to find $p_k^\dag$, which
may cause additional complexity compared to the \Alg{A}{K}.
Here we consider a simpler algorithm:
only the $k$-th weight is updated with fixing other weights
and normalize all weights, that is,

\paragraph{Update rule 2:}
\begin{align}
 \wcur{k}' &= \wcur{k} + \Delta w_k, \nonumber \\
 \wcur{k}'' &= \frac{\wcur{k}'}{\wcur{k}' + \sum_{l\ne k} \wcur{l}} =
 \frac{\wcur{k}'}{1+\Delta w_k}, \nonumber \\
 \wcur{l}'' &= \frac{\wcur{l}}{1 + \Delta w_k},\quad l\ne k.
\label{eq:up2}
\end{align}

This update does not require the computation of $p_k^\dag$.
We examine the relation between Update 1 and 2.
Assuming $\Delta w_k$ is sufficiently small, the Update 2 is approximated by
\begin{equation}
\label{eq:up2appk}
 \wcur{k}'' \simeq \frac{\wcur{k} + \Delta w_k}{1 + \Delta w_k} \simeq
  (\wcur{k} + \Delta w_k) (1 - \Delta w_k) \simeq \wcur{k} + (1-\wcur{k})\Delta w_k,
\end{equation}
 \begin{equation}
\label{eq:up2appl}
  w_l'' \simeq (1-\Delta w_k)\wcur{l} = \wcur{l} - \Delta w_k \wcur{l}, \quad l\ne k,
 \end{equation}
which means that the Update rule 2 is equivalent to the Update rule 1 where
the learning constant is shortened by a factor $1-\wcur{k}$.

Therefore, we see that if the Update rule 1 is locally stable,
the Update rule 2 is also locally stable.

In a similar way with component-wise algorithm, 
we can obtain one-side component-wise algorithm
for general $K$ based on Update rule 2.

\begin{algorithm}
\caption{Geometrical \Alg{C}{K,L} Component-wise one-side} 
\begin{algorithmic}[1]
 \State \textbf{Initialize} $\{w_k\}_{k=1,\ldots,K}$
   s.t.\ $\sumkK w_k = 1, w_k > 0$
 \Repeat
 \For{$k=1,\ldots,K$}
 \For{$\mbox{count}=1$ to $L$}
  \State Update $w_k$ by \textbf{Update rule 2}
 \EndFor
 \EndFor
\Until{Stopping criterion is satisfied}
\State \Return $\mathbf{w}$
\end{algorithmic}
\end{algorithm}

Note that updating $w_k$ affects the value of other $w_l$ ($l\ne k$)
because of the normalization.

\begin{proposition}
If the condition of local stability for $K=2$ is satisfied for all $k$,
\Alg{C}{K,L} is locally stable for any $L\ge1$.
\end{proposition}

\subsection{Local stability of the \Alg{A}{K}}
\label{sec:conv}

Now we are ready to prove the local stability of \Alg{A}{K}.

The \Alg{C}{K,1} is a sequential algorithm, and we can 
construct corresponding ``batch'' version of
the algorithm.

\begin{algorithm}
\caption{Geometrical \Alg{Cb}{K}} 
\begin{algorithmic}[1]
 \State \textbf{Initialize} $\{w_k^{(0)}\}_{k=1,\ldots,K}$
   s.t.\ $\sumkK w_k^{(0)} = 1, w_k^{(0)} > 0, t:=0$
 \Repeat
 \For{$k=1$ to $K$}
 Calculate the amount of
       changing values $\Delta w_{l(k)}$ of weight $w_l$ in the update
       of $w_k$ (Eq.(\ref{eq:up2})).
 \EndFor
 \State Calculate $\gamma_k$ by (\ref{eq:gammak}), where
       $\tval{\qcur}=\sum_{i=1}^K w_k^{(t)} \tval{p_k}$ 
 \State Update the weights by
       \begin{equation}
	w_k' = w_k^{(t)} + \sumlK \Delta w_{k(l)}, \quad w_l'=w_l^{(t)}, l\ne k
       \end{equation}
 \State Normalize $w_k'$
       \begin{equation}
	w_k^{(t+1)} = \frac{w_k'}{\sumkK w_k'}
       \end{equation}
 \State $t:=t+1$
\Until{Stopping criterion is satisfied}
\State \Return $\mathbf{w}$
\end{algorithmic}
\end{algorithm}

Since \Alg{Cb}{K} updates the weights by $K$ perturbations,
the condition for local stability is changed.
The following lemma gives a sufficient condition.
\begin{lemma}
\Alg{Cb}{K} is locally stable if $f$ satisfies
\begin{equation}
\label{eq:boundK2}
\dgamma0 < \frac{2}{K \max_k w_k^* (1-w_k^*) g(w_k^*)},
\end{equation}
 where $w_k^*$ denotes the optimal value.
\end{lemma}
This bound is given by multiplying $1/K$ to (\ref{eq:bound}), but it might
be very strict, because (\ref{eq:bound}) for $K=2$ is a better bound.
Further, as wee see,
\Alg{Cb}{K} is very similar to \Alg{C}{K,1}, where only the difference
is whether the former is a simultaneous update and the latter is a sequential update.
It is an open problem to obtain a better bound for $K>2$.

\begin{proof}
By the update of $w_k$, suppose the weight is changed from
 $\hat{\mathbf{w}}$ to $\hat{\mathbf{w}} + \Delta \mathbf{w}_k$,
 where $\sumlK (\Delta \mathbf{w}_k)_l=0$ because of the weight constraint.
If $f$  satisfies the condition (\ref{eq:bound}), it holds
 \begin{equation}
  \|\hat{\mathbf{w}}-\mathbf{w}^*\|^2 > 
  \|\hat{\mathbf{w}}+\Delta \mathbf{w}_k -\mathbf{w}^*\|^2
 \end{equation}
for all $k$.
By multiplying $1/K$ to the value of
 $df(0)/d\gamma$, the change of weights becomes
 $\hat{\mathbf{w}} + \Delta \mathbf{w}_k / K$
 in terms of the first order approximation, and
 the simultaneous update of the all weight, it becomes
 $\hat{\mathbf{w}} + \sumkK \Delta \mathbf{w}_k / K$.
Therefore, the new weight satisfies
\begin{align}
  \|\hat{\mathbf{w}}+\sumkK\frac{1}{K}\Delta \mathbf{w}_k -\wtruebf{}\|^2
 &=  \|\frac{1}{K}\sumkK(\hat{\mathbf{w}}+\Delta \mathbf{w}_k
 -\wtruebf{})\|^2 \nonumber \\
 &= \frac{1}{K^2} \|\sumkK(\hat{\mathbf{w}}+\Delta \mathbf{w}_k
 -\wtruebf{})\|^2 \nonumber \\
 &< \frac{1}{K^2} K^2 \max_k \|\hat{\mathbf{w}}+\Delta \mathbf{w}_k
 -\wtruebf{}\|^2 \nonumber \\
 &<   \|\hat{\mathbf{w}}-\wtruebf{}\|^2,
\end{align}
 which shows local stability of \Alg{Cb}{K}.
\end{proof}

The main theorem \ref{th:main} is proved by showing equivalence between
\Alg{A}{K} and \Alg{Cb}{K} as follows.

\begin{lemma}
When the update amounts of $\{w_k\}$ are sufficiently small,
The \Alg{A}{K} is equivalent to \Alg{Cb}{K}.
\end{lemma}

\begin{proof}
By the \Alg{A}{K}, the weights are updated by
 \begin{equation}
  w_k' = w_k + \Delta w_k,
 \end{equation}
 \begin{align}
  w_k'' &= \frac{w_k'}{\sumlK w_l'} = \frac{w_k + \Delta w_k}{\sumlK
  (w_l + \Delta w_l)} \nonumber \\
  &= \frac{w_k + \Delta w_l}{1 + \sumlK \Delta w_l} \nonumber \\
  &\simeq (w_k + \Delta w_k) (1 - \sumlK \Delta w_l) \nonumber \\
  &\simeq w_k + \Delta w_k - \sumlK \Delta w_l w_k.
  \label{eq:AlgAupdate}
 \end{align}
On the other hand, the update of $w_k$ in the Update rule 2 for small
 $\Delta w_k$ is given by (\ref{eq:up2appk}) and (\ref{eq:up2appl}).
Therefore, the amount of change of $w_l$ for the update of $w_k$ is given by
 \begin{equation}
  \Delta w_{l(k)} = - \Delta w_k w_l,
 \end{equation}
 then summing up them and we have the update of
 of the \Alg{Cb}{K} by
 \begin{align}
  w_k''' &\simeq w_k + (1-w_k)\Delta w_k - \sum_{l\ne k} \Delta w_l w_k
  \nonumber \\
  &= w_k + \Delta w_k - \sumlK \Delta w_l w_k,
 \end{align}
which coincides the update of (\ref{eq:AlgAupdate}).
\end{proof}

\end{document}